\pdfoutput=1
\documentclass[letterpaper]{article} 
\usepackage[preprint]{aaai2027}  
\usepackage[hyphens]{url}  
\usepackage{graphicx} 
\usepackage{natbib}  
\usepackage{caption}  
\usepackage{amsmath,amssymb,amsfonts,amsthm,mathtools,bm}
\usepackage{xcolor}
\usepackage{booktabs}

\newtheorem{theorem}{Theorem}
\newtheorem{lemma}{Lemma}
\newtheorem{assumption}{Assumption}

\newtheorem{corollary}{Corollary}

\DeclareMathOperator{\Cov}{Cov}
\DeclareMathOperator{\tr}{tr}
\DeclareMathOperator{\diag}{diag}

\newcommand{\R}{\mathbb R}
\newcommand{\E}{\mathbb E}
\newcommand{\Normal}{\mathcal N}

\title{On the Identifiability of Controlled World Models}
\author{Xiangteng Zhang\textsuperscript{1}\equalcontrib, Yang Guan\textsuperscript{1}\equalcontrib, Bo Zhang\textsuperscript{2}, \\Hongyang Li\textsuperscript{3}, Ya-Qin Zhang\textsuperscript{4}, and Shengbo Eben Li\textsuperscript{1,5}\corresponding}
\affiliations{
	\textsuperscript{1}School of Vehicle and Mobility, Tsinghua University \quad
	\textsuperscript{2}Didi Voyager Labs, Didi Autonomous Driving \\
	\textsuperscript{3}School of Computing and Data Science, The University of Hong Kong \\
	\textsuperscript{4}Institute for AI Industry Research, Tsinghua University \quad
	\textsuperscript{5}College of AI, Tsinghua University
}

\begin{document}
	\maketitle
	
	\begin{abstract}
		World model serves as a promising tool to infer environment dynamics under high-dimensional observations and candidate actions.
		Recently, LeCun's JEPA provides a compelling framework for learning such models in representation space. 
		Its action-conditioned extension plays a central role in visual control and latent-space planning, but leaves a fundamental question: can it recover the controlled dynamics from nonlinear observations?
		This paper presents a joint identifiability condition for controlled world models with Gaussian latent states, which consists of two coupled components: (1) representation identifiability and (2) transition identifiability. 
		The former depends on the spectral separation property while the latter is related to non-degenerate variation of conditional action. 
		We prove that when this condition holds, minimizing the LeJEPA-style predictive objective can recover both latent states and controlled dynamics in the sense of orthogonal transformation. 
		We further prove that the upper bound of transition prediction error is inversely proportional to the spectral separation margin.
		We also characterize an attainable amplification of counterfactual prediction error that scales inversely with the weakest conditional action-excitation margin.
		The theoretical predictions are empirically supported across four nonlinear observation settings.
	\end{abstract}
	
	\section{Introduction}
\label{sec:introduction}

\begin{figure*}[t!]
	\centering
	\includegraphics[width=0.8\textwidth]{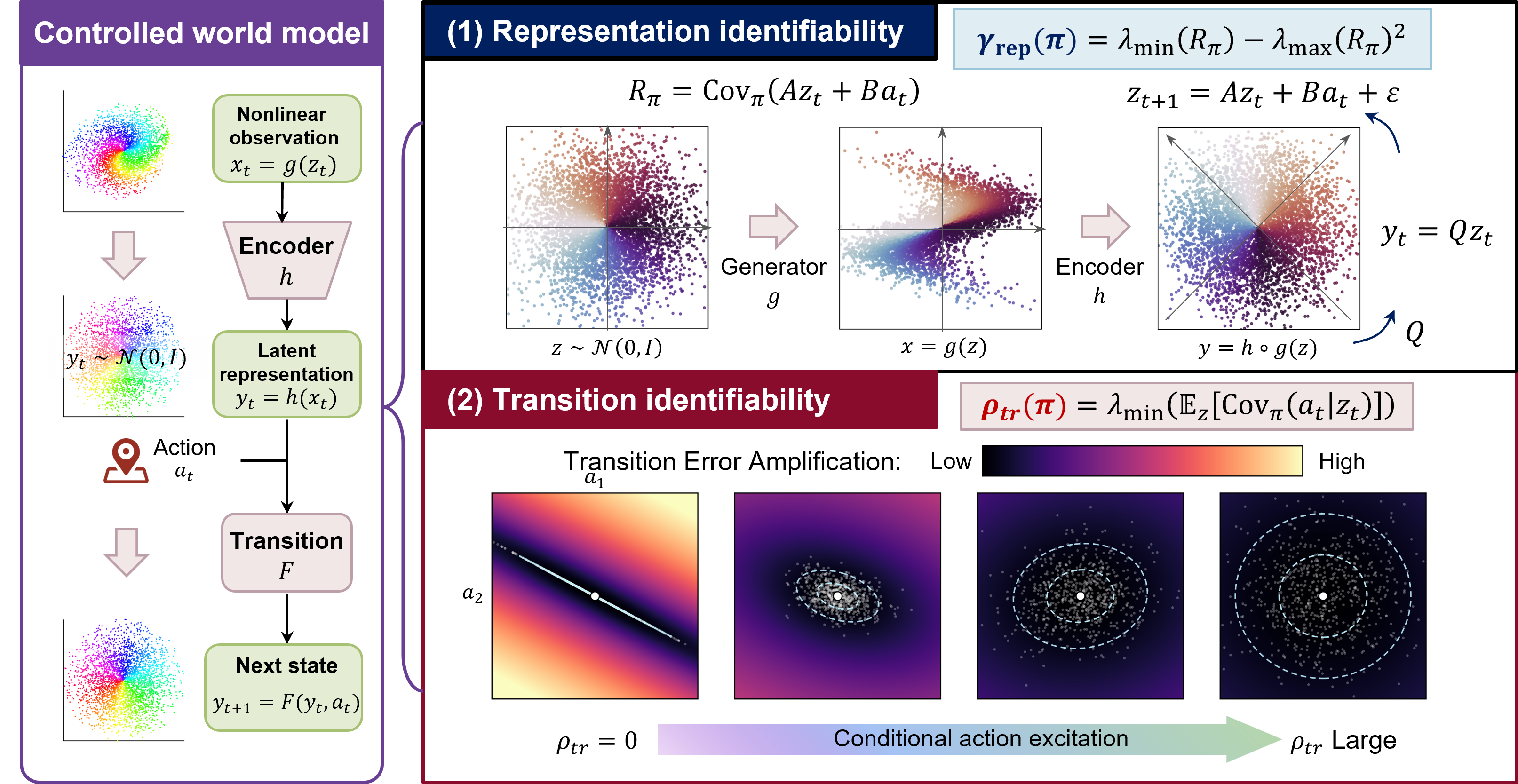}
	\caption{Overview of controlled world-model identifiability. An encoder maps nonlinear observations to a standard-Gaussian latent representation, and an action-conditioned predictor models the next representation. Representation identifiability is governed by a predictable-signal spectral margin, while transition identifiability additionally depends on the weakest conditional action excitation.}
	\label{fig:method-overview}
\end{figure*}

World models seek to capture how an environment evolves, enabling an agent to anticipate future states, simulate the consequences of its behavior, and support decision-making and planning through a latent predictive model. They have become increasingly important in developing physics-native intelligence, including embodied robotics, autonomous driving, and physical data generation. In general, world models encode images, videos, or sensory measurements into compact state representations and learn environmental dynamics in the latent space. Such representations can subsequently support behavior prediction, policy learning, value estimation, and trajectory optimization. In control settings, these capabilities require a model that captures how the future state varies across candidate actions.

Controlled world models are defined as a class of world models characterized by latent transition predictor conditioned on both the current latent state and the action candidates.
By explicitly relating states, actions, and future states, they provide the structure needed to compare candidate actions, evaluate counterfactual outcomes, and support model-based planning. 
A prominent framework for learning such models from high-dimensional observations is the Joint-Embedding Predictive Architecture (JEPA), which trains model parameters by minimizing the prediction error between two temporally related observations. Meanwhile, it uses regularization or distributional constraints on the latent space to prevent representational collapse \citep{lecun2022Path,assran2023SelfSupervised,bardes2024Revisiting,balestriero2025LeJEPA}. Recent methods such as LeWorldModel and V-JEPA 2-AC incorporate actions into latent prediction and have demonstrated promising performance in visual control, policy learning, and latent-space planning \citep{assran2025VJEPA, hafner2025Mastering,zhan2025Bootstrap,maes2026LeWorldModel}. 
Despite these empirical advances, existing results do not establish whether controlled world models can correctly identify both the latent state representation and the underlying transition dynamics. This question, which we refer to as the identifiability problem of controlled world models, remains theoretically underexplored.

Theoretical analysis of dynamical systems from data has advanced primarily in settings where system states can be directly observed. Recent work in data-driven control has examined controllability from sampled transitions \citep{yang2024Controllability}, stability certification from finite data \citep{yang2024Stability}, and canonical representations of transition data \citep{zhan2025Canonical}. These studies characterize system properties and transition structure from data, but presuppose a fixed and observed state coordinate system. Controlled world models do not have access to such coordinates because the latent state must itself be inferred from nonlinear observations. This additional ambiguity gives rise to two coupled identification problems. Representation identifiability concerns whether the encoder identifies the latent state, which is generally impossible without additional structure \citep{locatello2019Challenging,khemakhem2020Variational,hyvarinen2024Identifiability}. Under Gaussian assumptions, recent LeJEPA theory establishes representation identifiability up to an orthogonal transformation \citep{klindt2026When}, yet leaves the controlled transition unidentified. Transition identifiability concerns whether the predictor identifies the controlled dynamics. Classical system identification emphasizes the need for informative action variation, often formalized through persistent excitation \citep{willems2005Note}, but assumes direct access to state coordinates and therefore does not address representation identification. These two lines of theory address complementary parts of the problem, but neither establishes when action-conditioned latent prediction jointly identifies the state representation and controlled dynamics.

In this paper, we fill the gap by developing a systematic identifiability theory for controlled world models under nonlinear observations. Representation identifiability is characterized by a spectral separation property of the predictable signal induced by the behavior policy, whereas transition identifiability is governed by the non-degeneracy of conditional action variation given the current state. Under controlled spectral separation and sufficient conditional action excitation, we prove that every global minimizer of the population objective identifies the true latent state and the full controlled conditional-mean dynamics under a common orthogonal transformation. We then establish quantitative identifiability guarantees in the presence of encoder excess risk and predictor error. Our analysis further shows that conditional action excitation determines the strength of transition identification: weakly excited action directions can admit small training error while producing substantially larger errors in counterfactual action responses. Finally, we connect the identified transition model to goal-conditioned latent planning and empirically examine how action coverage affects representation identifiability, transition identifiability, counterfactual prediction, and planning across nonlinear observation maps and behavior policies.

Our main contributions are summarized as follows:
\begin{itemize}
	\item We establish sufficient conditions for joint representation and transition identifiability in controlled world models with invertible nonlinear observations and linear Gaussian latent dynamics. Representation identifiability is guaranteed under a spectral separation condition jointly determined by the controlled dynamics and behavior policy, whereas transition identification over the full state-action space is guaranteed under positive conditional action excitation. Together, these conditions guarantee recovery of the latent state and the controlled conditional mean up to a common orthogonal transformation.
	
	\item We derive quantitative identification guarantees under imperfect optimization by using a Hermite decomposition of the learned representation, and show that the upper bounds on transition prediction error scale inversely with the spectral separation margin, while predictor approximation error enters the transition bound additively.
	
	\item We characterize the attainable amplification of counterfactual prediction error by perturbing the predictor along the least excited action direction, and show that the amplification factor is inversely proportional to the weakest conditional action-excitation margin.
\end{itemize}

	\section{Related Work}
\label{sec:related-work}

\paragraph{Latent world models for planning.}
World models have been widely studied as internal predictive models for planning and control. 
Early neural world models learned compact latent dynamics for imagination-based policy learning \citep{ha2018World,hafner2019Learning,mu2020Mixed}, while later model-based reinforcement learning methods developed probabilistic latent dynamics, value-aware dynamics, and model-predictive control \citep{hansen2022Temporal,hansen2024TDMPC2,zhan2025Bootstrap, li2023Reinforcement} in learned state spaces \citep{hafner2020Dream,hafner2025Mastering,schrittwieser2020Mastering}. 
These results establish that learned latent dynamics can be useful for planning, but not that the learned state or controlled transition is identifiable from nonlinear observations. 
Our work targets this missing structural question: when does a latent world model identify both the state and the controlled conditional-mean transition?

\paragraph{Joint embedding predictive learning and latent prediction.}
Joint embedding predictive architectures learn by predicting target representations rather than reconstructing pixels, making them attractive for representation learning and latent world modeling \citep{assran2025VJEPA,bardes2024Revisiting}. 
Recent world models extend this principle to latent dynamics without pixel reconstruction, while using regularization or target encoders to prevent representational collapse \citep{balestriero2025LeJEPA,maes2026LeWorldModel}. 
Under Gaussian latent assumptions, recent theory for LeJEPA shows that prediction from passive observation pairs can recover latent states up to an orthogonal transformation, but does not establish identifiability of dynamics conditioned on actions \citep{klindt2026When}. 
We study this controlled setting and show that, when data are generated by a behavior policy, representation recovery and controlled transition recovery are characterized by distinct identifiability margins.

\paragraph{Action-conditioned and latent-action world models.}
A growing body of work incorporates actions into latent world models so that different candidate actions induce different predicted futures. 
This includes action-conditioned predictors for robot planning, end-to-end latent world models with action-sensitive objectives, and latent-action models that infer abstract actions from videos or visual feature differences \citep{assran2025VJEPA,maes2026LeWorldModel}. 
These methods reflect an important shift from passive future prediction to controlled prediction, where the model is expected to represent how the world responds to alternative actions. 
However, action conditioning is usually justified empirically, through action sensitivity, rollout quality, or downstream planning performance. 
The gap is not whether actions can be inserted into latent predictors, but when action-conditioned prediction makes the action effect itself identifiable.

\paragraph{Identifiability and system identification.}
Identifiability is a central difficulty in nonlinear representation learning. 
Classical results in nonlinear ICA and unsupervised disentanglement show that latent variables are generally not identifiable from nonlinear observations without additional assumptions or structural signals \citep{hyvarinen1999Nonlinear,locatello2019Challenging}. 
Positive identifiability results introduce such structure through temporal dependence, auxiliary variables, conditional latent distributions, or interventions \citep{hyvarinen2016Unsupervised,hyvarinen2019Nonlinear,khemakhem2020Variational,scholkopf2021Causal,vonkugelgen2023Nonparametric,varici2024General}. 
In parallel, classical system identification emphasizes that identifying controlled dynamics requires sufficiently informative inputs, often formalized through persistent excitation or related rank conditions \citep{vanoverschee1996Subspace,willems2005Note}. 
These literatures cover complementary pieces of the problem, but neither directly gives identifiability for controlled world models: representation-learning results usually do not identify controlled transitions, while system-identification results often assume observed or linearly measured states. 
Our analysis combines the two perspectives under nonlinear observations, showing when sufficient conditional action excitation and a representation-identifiability spectral condition make the representation and controlled conditional-mean transition jointly identifiable.

	\section{Method}
\label{sec:method}

\subsection{Problem Formulation}
\label{subsec:problem-formulation}

We formulate controlled latent prediction to separate two identification problems that are entangled in observations: identifying a valid state coordinate and identifying how that state responds to actions. An action-conditioned predictor is necessary because a state-only predictor sees only the transition averaged over the behavior policy; however, conditioning on the action is informative only to the extent that the data contain action variation conditional on the state.

\paragraph{World and behavior distribution.}
Let \(z_t\in\R^d\) be the latent state and \(x_t=g(z_t)\) its observation under an unknown invertible map \(g\). Invertibility isolates the representation problem from information loss in the observation process: the learner must undo a nonlinear change of coordinates, but the observation still contains the full state. Trajectories are collected under a behavior policy \(\pi(a_t\mid z_t)\), which, together with the environment, induces the training distribution \(P_\pi\). Making this distribution explicit is essential because transition identification depends on the policy's conditional action support.

We use a stationary linear--Gaussian latent system as the minimal setting in which these effects can be analyzed exactly. Gaussianity gives a tractable spectral decomposition and is compatible with the standard-Gaussian representation constraint; linear controlled dynamics keep state and action effects separate; and stationarity provides a common marginal distribution across consecutive representations.

\begin{assumption}[Jointly Gaussian behavior distribution]
	\label{ass:joint-gaussian-policy}
	The latent state satisfies \(z_t\sim\Normal(0,I_d)\), and the behavior policy induces a centered jointly Gaussian distribution \(P_\pi(z_t,a_t)\). The state and action may be dependent, and the conditional action covariance may be singular.
\end{assumption}

\begin{assumption}[Stationary controlled transition]
	\label{ass:controlled-transition}
	The latent dynamics are \(z_{t+1}=Az_t+Ba_t+\xi_t\), where \(\xi_t\) is zero-mean Gaussian noise independent of \((z_t,a_t)\). The process is stationary with \(z_{t+1}\sim\Normal(0,I_d)\).
\end{assumption}

\paragraph{Learner.}
The learner consists of an encoder \(h:\mathcal X\to\R^d\) and a deterministic action-conditioned predictor \(F\). We study its population objective, defined as the expected prediction loss under the stationary data-generating distribution \(P_\pi\) induced by the behavior policy:
\begin{equation}
	\mathcal L_\pi(h,F)
	=
	\E_{P_\pi}
	\left[
	\|h(x_{t+1})-F(h(x_t),a_t)\|_2^2
	\right].
	\label{eq:loss}
\end{equation}
The squared loss makes the predictor target the conditional mean in representation space rather than the full conditional distribution. We constrain \(h(x_t)\sim\Normal(0,I_d)\), as in LeJEPA-style Gaussian regularization, to prevent collapse and fix the representation scale while retaining the unavoidable rotational symmetry of the latent distribution.

\begin{assumption}[Predictor richness and continuity]
	\label{ass:predictor-richness}
	For every admissible encoder \(h\), the predictor class can realize the optimal conditional-mean predictor
	\(F_h^\star(h(x_t),a_t):=\E[h(x_{t+1})\mid h(x_t),a_t]\).
	The learned predictor is continuous in its representation and action inputs.
\end{assumption}

\subsection{Identifiability of Controlled World Models}
\label{subsec:state-transition-identifiability}

\paragraph{Representation identifiability.}
We define representation identifiability as the existence of an orthogonal matrix \(Q\in O(d)\) such that
\[
h(g(z))=Qz.
\]
This defines latent-state identifiability up to an orthogonal transformation.

\paragraph{Transition identifiability.}
We define transition identifiability under \(P_\pi\) as
\[
F(y,a)=QAQ^\top y+QBa
\]
This definition can also be denoted as \(F(h(g(z)),a)=Q(Az+Ba)\). This concerns the controlled conditional mean in the identified coordinates.

\paragraph{Identification margins.}
Representation identifiability is governed by the predictable signal available under the behavior policy:
\begin{equation}
	\begin{aligned}
		R_\pi
		&:=\Cov_{P_\pi}\!\left(\E[z_{t+1}\mid z_t,a_t]\right)
		=\Cov_{P_\pi}(Az_t+Ba_t),\\
		\gamma_{\mathrm{rep}}(\pi)
		&:=\lambda_{\min}(R_\pi)-\lambda_{\max}(R_\pi)^2 .
	\end{aligned}
	\label{eq:representation_margin}
\end{equation}
Stationarity gives \(\Cov(\xi_t)=I_d-R_\pi\succeq0\). The representation margin \(\gamma_{\mathrm{rep}}(\pi)\) compares the weakest first-order predictable direction with the strongest higher-order nonlinear alternative.

Transition identifiability is governed by the action variation that remains after conditioning on the current state:
\begin{equation}
	\begin{aligned}
		\Sigma_{\mathrm{tr}}(\pi)
		&:=\E_{z_t}\!\left[\Cov_\pi(a_t\mid z_t)\right],\\
		\rho_{\mathrm{tr}}(\pi)
		&:=\lambda_{\min}\!\left(\Sigma_{\mathrm{tr}}(\pi)\right).
	\end{aligned}
	\label{eq:transition_margin}
\end{equation}
The transition margin \(\rho_{\mathrm{tr}}(\pi)\) is the weakest conditional action excitation. Under the jointly Gaussian behavior model, \(\Cov_\pi(a_t\mid z_t)\) does not depend on \(z_t\); hence, \(\rho_{\mathrm{tr}}(\pi)>0\) gives full conditional support over the action space.

The two margins address different failure modes. A positive \(\gamma_{\mathrm{rep}}(\pi)\) separates the latent coordinates from nonlinear predictive features. A positive \(\rho_{\mathrm{tr}}(\pi)\) distinguishes the controlled response in every action direction and gives full support to the joint Gaussian state-action distribution. Together with continuity of \(F\), this support extends transition identification from an almost-sure statement under \(P_\pi\) to the entire state-action space.

We can now state the main identification result. In the Gaussian setting, the eigenvalues of \(R_\pi\) determine the predictability of the first-order latent components, whereas any higher-order component has predictable energy at most \(\lambda_{\max}(R_\pi)^2\). A positive representation margin therefore makes every true latent direction more predictable than any nonlinear alternative.

\begin{theorem}[Identifiability of state and controlled transition]
	\label{thm:state-transition-identifiability}
	Suppose Assumptions~\ref{ass:joint-gaussian-policy}--\ref{ass:predictor-richness} hold and the observation map \(g\) is invertible. Consider the objective in Eq.~\eqref{eq:loss} over encoders satisfying \(h(x_t)\sim\Normal(0,I_d)\). If
	\[
	\gamma_{\mathrm{rep}}(\pi)>0
	\qquad\text{and}\qquad
	\rho_{\mathrm{tr}}(\pi)>0,
	\]
	then every global minimizer \((h,F)\) of \(\mathcal L_\pi\) identifies the latent state and the controlled conditional mean up to an orthogonal transformation. Specifically, there exists \(Q\in O(d)\) such that
	\[
	\begin{aligned}
		&h(g(z_t))=Qz_t,\\
		&F(h(g(z_t)),a_t)=Q(Az_t+Ba_t),\quad \text{a.s.}
	\end{aligned}
	\]
	Moreover, in the identified coordinates, the predictor satisfies
	\(F(y,a)=QAQ^\top y+QBa\) for every
	\((y,a)\in\R^d\times\R^m\).
	
	The optimal objective value is
	\(\operatorname{tr}(\Cov(\xi_t))=d-\operatorname{tr}(R_\pi)\).
	Conversely, for any \(Q\in O(d)\), the encoder defined by
	\(h_Q(g(z))=Qz\) and the predictor
	\(F_Q(y,a)=QAQ^\top y+QBa\) attain this minimum.
\end{theorem}

The theorem separates two identification problems that are often conflated. The representation margin rules out nonlinear reparameterizations and identifies the latent state up to an orthogonal transformation. Conditional action excitation then identifies how that state responds to every action. Positive excitation gives full state-action support; together with continuity of \(F\), this upgrades an almost-sure on-policy statement to identification over the entire state-action space.

\paragraph{Proof sketch.}
For a fixed encoder, squared loss is minimized by the conditional mean of \(h(x_{t+1})\) given \((h(x_t),a_t)\). Minimizing prediction risk is thus equivalent to maximizing the predictable energy of the representation. Expand \(h\circ g\) in the Gaussian Hermite basis. First-order components are governed by the eigenvalues of \(R_\pi\), while every higher-order component contributes at most \(\lambda_{\max}(R_\pi)^2\) times its variance. When \(\gamma_{\mathrm{rep}}(\pi)>0\), an optimum can therefore contain only first-order components. The Gaussian constraint makes the resulting linear map orthogonal, so \(h(g(z))=Qz\). Substitution into the dynamics yields the predictor \(Q(Az+Ba)\); positive conditional excitation and continuity extend the identity to every \((y,a)\in\R^d\times\R^m\). The full proof appears in the first proof section of the appendix.

\subsection{Approximate Identifiability}
\label{subsec:quantitative-identifiability}

The preceding result assumes exact global optimality. To quantify graceful degradation away from this ideal, let
\(F_h^\star(h(x_t),a_t):=\E[h(x_{t+1})\mid h(x_t),a_t]\)
denote its optimal conditional-mean predictor, and let
\(\mathcal L^\star:=d-\tr(R_\pi)\) denote the minimum risk in Theorem~\ref{thm:state-transition-identifiability}. We separate error due to the representation from error due to the learned predictor by defining
\(\Delta_{\rm enc}:=\mathcal L_\pi(h,F_h^\star)-\mathcal L^\star\)
and the predictor-side error as
\(\Delta_{\rm pred}:=\E_{P_\pi}\|F(h(x_t),a_t)-F_h^\star(h(x_t),a_t)\|^2\).

\begin{theorem}[Quantitative identifiability under the behavior policy]
	\label{thm:quantitative-identifiability}
	Suppose the conditions of Theorem~\ref{thm:state-transition-identifiability} hold except that \((h,F)\) need not be a global minimizer of \(\mathcal L_\pi\). There exist universal constants \(c,C>0\) such that, if \(h(x_t)\sim\Normal(0,I_d)\) and
	\(\Delta_{\rm enc}\le c\,\gamma_{\mathrm{rep}}(\pi)\), then there exists \(Q\in O(d)\) such that, with
	\(\eta:=\Delta_{\rm enc}/\gamma_{\mathrm{rep}}(\pi)\),
	\[
	\E\|h(g(z_t))-Qz_t\|^2
	\le C\eta,
	\]
	Moreover, defining
	\(\varepsilon_F(z,a):=F(h(g(z)),a)-Q(Az+Ba)\), we have
	\begin{equation}
		\E_{P_\pi}\!\left[\|\varepsilon_F(z_t,a_t)\|^2\right]
		\le C\Delta_{\rm pred}
		+C\bigl(1+\|A\|_{\rm op}^2\bigr)\eta.
	\end{equation}
\end{theorem}

The two errors affect different parts of the learned world model. The encoder excess risk controls deviation from an orthogonal copy of the latent state; the predictor error captures the additional cost of using \(F\) rather than the optimal predictor for that representation. The factor \(1/\gamma_{\mathrm{rep}}(\pi)\) exposes the conditioning of approximate representation identification: as the spectral gap narrows, nonlinear features become nearly as predictive as the true latent directions.

\paragraph{Proof sketch.}
The risk attained by \(F_h^\star\) is determined by the predictable energy of \(h(x_{t+1})\). Its Hermite expansion shows that the representation margin bounds the total variance in higher-order components by \(O(\Delta_{\rm enc}/\gamma_{\mathrm{rep}}(\pi))\). The remaining first-order component is close to an orthogonal map because \(h(x_t)\sim\Normal(0,I_d)\), which gives the representation bound. Combining this error with the deviation of \(F\) from \(F_h^\star\) gives the predictor bound. The full proof is provided in the corresponding appendix section.

\subsection{Counterfactual Error Amplification}
\label{subsec:counterfactual-amplification}

Theorem~\ref{thm:quantitative-identifiability} provides an in-distribution transition-error guarantee under \(P_\pi\). Planning and control, however, require predictions for actions that may be rare under this distribution, so small on-policy error need not imply reliable counterfactual prediction. We characterize this gap in the identified coordinates \(y=Qz\), where the true conditional-mean predictor is \(F^\star(y,a)=QAQ^\top y+QBa\).

For a predictor \(F\), define its on-policy excess risk relative to \(F^\star\) as
\begin{equation}
	\begin{aligned}
		\mathcal E_\pi(F)
		&:=\E_{P_\pi}\!\left[\|Qz_{t+1}-F(Qz_t,a_t)\|^2\right]\\
		&\quad-\E_{P_\pi}\!\left[\|Qz_{t+1}-F^\star(Qz_t,a_t)\|^2\right].
	\end{aligned}
	\label{eq:on-policy-excess-risk}
\end{equation}
Let \(\mu_\pi(z):=\E_\pi[a_t\mid z_t=z]\). To isolate the role of conditional coverage, consider a counterfactual distribution with the same conditional mean but identity covariance: \(a_{\rm cf}=\mu_\pi(z)+\eta_{\rm cf}\), where \(\eta_{\rm cf}\sim\Normal(0,I_m)\) is independent of \(z\). Its transition error is
\begin{equation}
	\mathcal E_{\rm cf}(F)
	:=
	\E
	\left[
	\|F(Qz,a_{\rm cf})
	-F^\star(Qz,a_{\rm cf})\|^2
	\right].
	\label{eq:counterfactual-transition-error}
\end{equation}

For this counterfactual analysis, we additionally assume that the predictor class contains
\(F^\star(y,a)+D(a-\mu_\pi(Q^\top y))\) for every matrix
\(D\in\R^{d\times m}\).

\begin{theorem}[Counterfactual error amplification]
	\label{thm:counterfactual-amplification}
	Suppose the conditions of Theorem~\ref{thm:state-transition-identifiability} and the additional predictor-class condition above hold. For every \(\delta>0\), the predictor class contains a continuous predictor \(F_\delta\) satisfying
	\[
	\mathcal E_\pi(F_\delta)=\delta,
	\qquad
	\mathcal E_{\rm cf}(F_\delta)
	=
	\frac{\delta}{\rho_{\mathrm{tr}}(\pi)}.
	\]
\end{theorem}

The theorem shows that exact transition identifiability does not imply well-conditioned approximate identification. When \(\rho_{\mathrm{tr}}(\pi)>0\), every action direction has nonzero conditional variation, so the controlled transition is identified at a global optimum. However, when \(\rho_{\mathrm{tr}}(\pi)\) is small, the behavior policy provides little variation along the least-excited action direction, and the training objective weakly penalizes errors in the corresponding action response. The constructed rank-one perturbation exploits precisely this direction: it incurs on-policy excess risk \(\delta\), while its counterfactual error is \(\delta/\rho_{\mathrm{tr}}(\pi)\). Thus, \(\rho_{\mathrm{tr}}(\pi)\) acts as a conditioning parameter for approximate transition identification. This construction provides an attainable lower bound on counterfactual error amplification, rather than a uniform upper bound over the predictor class.

\paragraph{Proof sketch.}
The conditional-mean orthogonality identity rewrites the on-policy excess risk as the mean squared deviation from \(F^\star\). Let \(v_{\min}\) be a unit eigenvector of \(\Sigma_{\mathrm{tr}}(\pi)\) with eigenvalue \(\rho_{\mathrm{tr}}(\pi)\), and perturb \(F^\star\) in this action direction by a rank-one map scaled by \(\sqrt{\delta/\rho_{\mathrm{tr}}(\pi)}\). Under the behavior policy, the residual action has covariance \(\Sigma_{\mathrm{tr}}(\pi)\), so the perturbation contributes exactly \(\delta\); under the counterfactual distribution its covariance is \(I_m\), giving error \(\delta/\rho_{\mathrm{tr}}(\pi)\). The full proof is provided in the Appendix \ref{app:proof-counterfactual-amplification}.

At \(\rho_{\mathrm{tr}}(\pi)=0\), some action direction has no residual variation once the state is fixed. Predictors may then agree exactly on the behavior distribution while disagreeing under counterfactual actions, making the controlled transition structurally non-identifiable outside the behavior support.

For the following state-dependent Gaussian policy, we show that the transition margin reduces directly to the exploration-noise variance.
\begin{corollary}[Exploration-noise scaling]
	\label{cor:exploration-noise-scaling}
	Consider the behavior-policy family
	\(a_t=\sqrt{1-\sigma^2}\,Kz_t+\sigma\eta_t\), where
	\(\eta_t\sim\Normal(0,I_m)\), \(\eta_t\perp z_t\), and
	\(\sigma\in[0,1]\). Then
	\[
	\Sigma_{\mathrm{tr}}(\pi_\sigma)=\sigma^2I_m,
	\qquad
	\rho_{\mathrm{tr}}(\pi_\sigma)=\sigma^2.
	\]
	Consequently, for every \(\sigma>0\) and \(\delta>0\), the predictor class contains a continuous predictor \(F_\delta\) satisfying
	\(\mathcal E_{\pi_\sigma}(F_\delta)=\delta\) and
	\(\mathcal E_{\rm cf}(F_\delta)=\delta/\sigma^2\).
	Thus the construction attains a counterfactual amplification factor of \(1/\sigma^2\). At
	\(\sigma=0\), the behavior policy is deterministic given the state, and the controlled transition is not identifiable outside its support.
\end{corollary}

Thus, any positive exploration noise gives full conditional action support, yet transition identification becomes increasingly ill-conditioned as \(\sigma\) decreases. The attainable factor in this construction grows as \(1/\sigma^2\) and diverges in the deterministic-policy limit.

	\section{Experiments}
\label{sec:experiments}

The experiments test three implications of the theory through controlled comparisons. The representation-margin sweep varies spectral separation while maintaining non-degenerate conditional action excitation. The behavior-policy sweep varies conditional action excitation while keeping the representation regime fixed. The planning evaluation then measures how counterfactual transition errors affect downstream action selection. Together, these experiments distinguish the roles of \(\gamma_{\rm rep}\) and \(\rho_{\rm tr}\) and evaluate their consequences for control.

\subsection{Representation Identifiability}
\label{subsec:experimental-identification}

To examine how spectral separation affects representation identifiability, actions are sampled independently from a standard Gaussian, fixing \(\rho_{\rm tr}=1\), while the predictable-signal spectrum is varied across the sufficient-condition boundary \(\gamma_{\rm rep}=0\). Conditional action excitation therefore remains non-degenerate throughout, so changes in identification quality can be attributed to the representation regime. The sweep is repeated for all four observation maps.

Spectral separation consistently improves identification across observation maps. In the identifiable regime, the encoder inverts each nonlinear observation map and identifies the latent polar structure up to an approximately orthogonal transformation (Figure~\ref{fig:cwm-representation-geometry}). As shown in Figure~\ref{fig:cwm-identification-curves}, both representation and controlled transition errors decrease as the weakest predictable direction strengthens. By fixing the action excitation, this joint improvement supports the predicted role of spectral separation in identifying the latent coordinates and the dynamics expressed in those coordinates. Notably, the estimation error changes smoothly through \(\gamma_{\rm rep}=0\), indicating that the theorem is providing only a sufficient condition.

\begin{figure}[htbp]
	\centering
	\includegraphics[width=\linewidth]{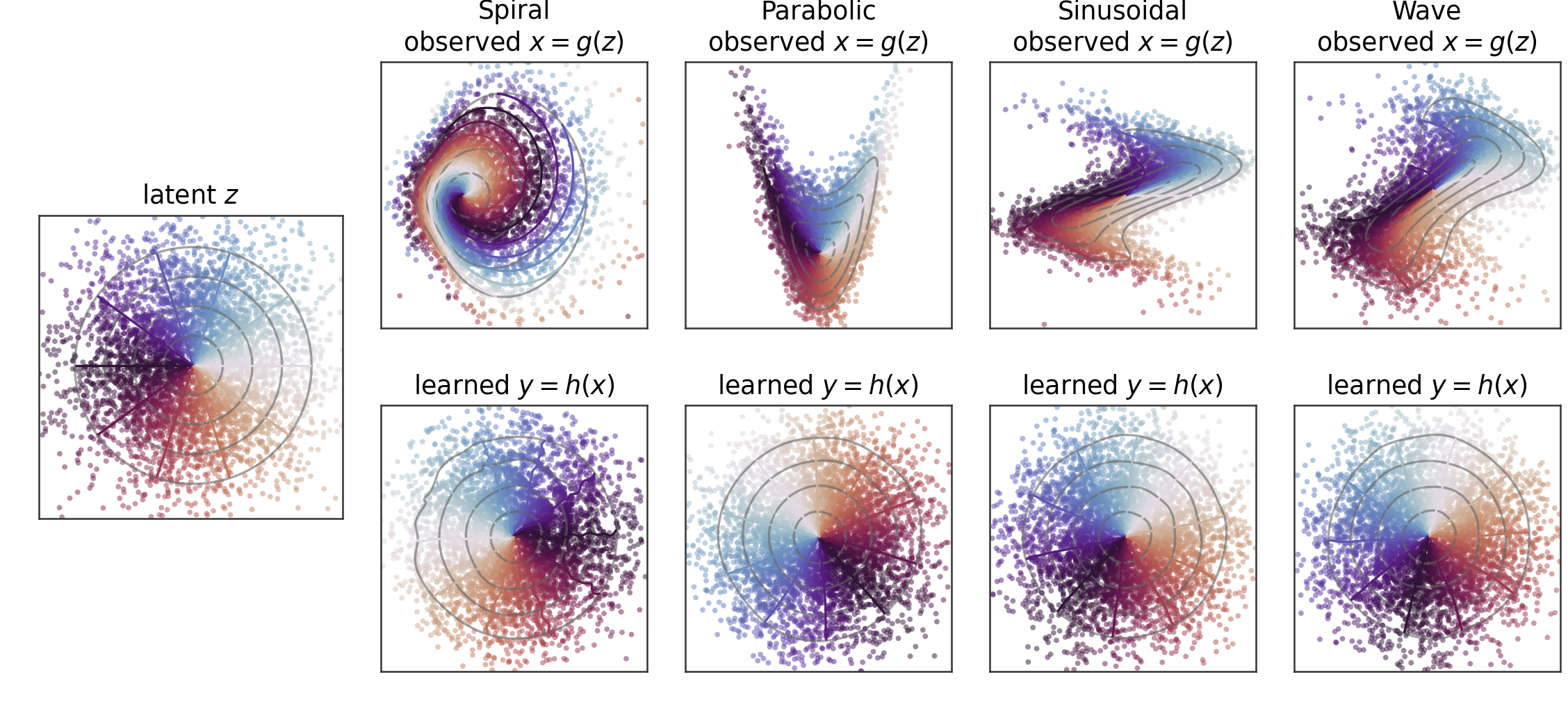}
	\caption{Representation geometry across nonlinear observation maps. The left panel shows the latent state \(z\); each remaining column shows the observation \(x=g(z)\) above and the learned representation \(y=h(x)\) below. Shared colors and polar calibration curves track corresponding latent locations.}
	\label{fig:cwm-representation-geometry}
\end{figure}

\begin{figure}[htbp]
	\centering
	\includegraphics[width=0.49\linewidth]{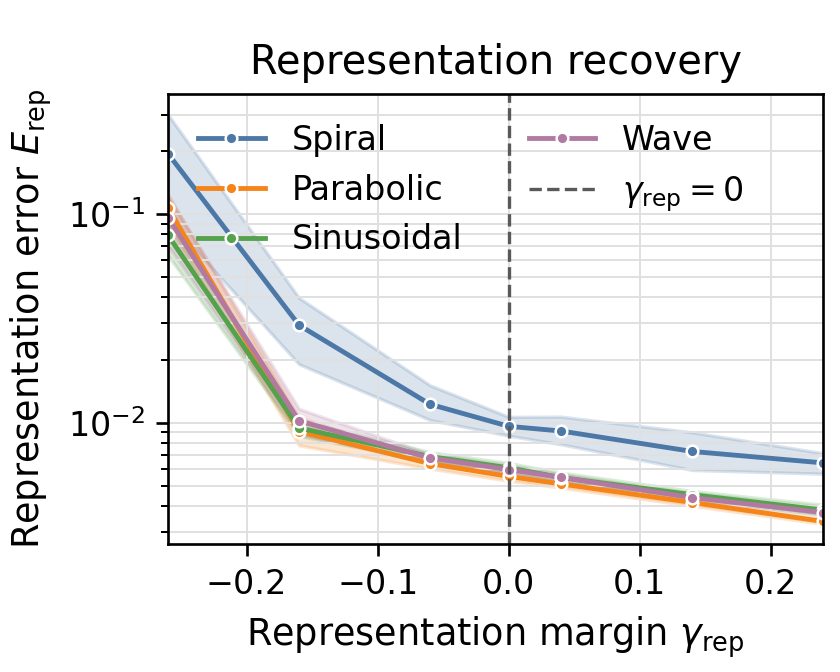}
	\hfill
	\includegraphics[width=0.49\linewidth]{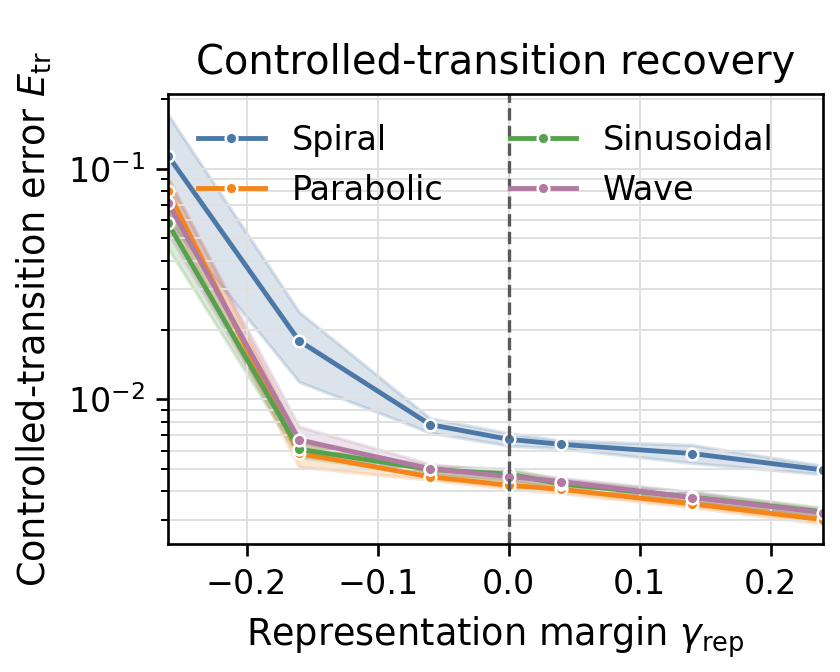}
	\caption{Identification performance as the representation margin \(\gamma_{\rm rep}\) increases. Left: normalized representation error \(E_{\rm rep}\); right: normalized controlled-transition error \(E_{\rm tr}\). Lines and shading show means and 95\% confidence intervals over five runs.}
	\label{fig:cwm-identification-curves}
\end{figure}

\subsection{Transition Identifiability and Counterfactual Prediction}
\label{subsec:behavior-coverage}
The transition experiment keeps the representation regime and controlled conditional mean fixed while varying only the conditional action excitation induced by the behavior policy. The policy family preserves a standard-Gaussian action marginal while satisfying
\begin{equation}
	\Cov(a_t\mid z_t)=\sigma^2I,
	\qquad
	\rho_{\rm tr}(\pi_\sigma)=\sigma^2.
	\label{eq:coverage-transition-margin}
\end{equation}
Thus, \(\sigma\) changes conditional coverage without changing the marginal action scale. At \(\sigma=0\), the data constrain only the closed-loop map \(A+BK\), while positive \(\sigma\) restores full conditional support, though identification remains poorly conditioned under weak excitation. The representation margin stays positive throughout. We compare behavior-action error with error under a broader intervention distribution and separately estimate the state and action components of the learned transition. Unlike Theorem~\ref{thm:counterfactual-amplification}, the experimental counterfactual distribution is fixed across coverage levels, so the experiment tests the predicted amplification trend rather than the exact theoretical factor. Full details are given in the supplemental materials.

Conditional action coverage determines whether behavioral accuracy extends to alternative actions. At \(\sigma=0\), prediction remains accurate on behavior actions, while counterfactual error is large and the state and action components are not separately identifiable (Figures~\ref{fig:cwm-behavior-coverage} and~\ref{fig:cwm-quantitative-summary}). Increasing conditional excitation improves counterfactual prediction and both transition components. Since representation error remains low and \(\gamma_{\rm rep}>0\) throughout, these improvements are attributable to transition excitation rather than representation quality. Thus, action conditioning and low behavioral prediction error alone do not guarantee reliable responses to alternative controls.

\begin{figure}[htbp]
	\centering
	\includegraphics[width=\linewidth]{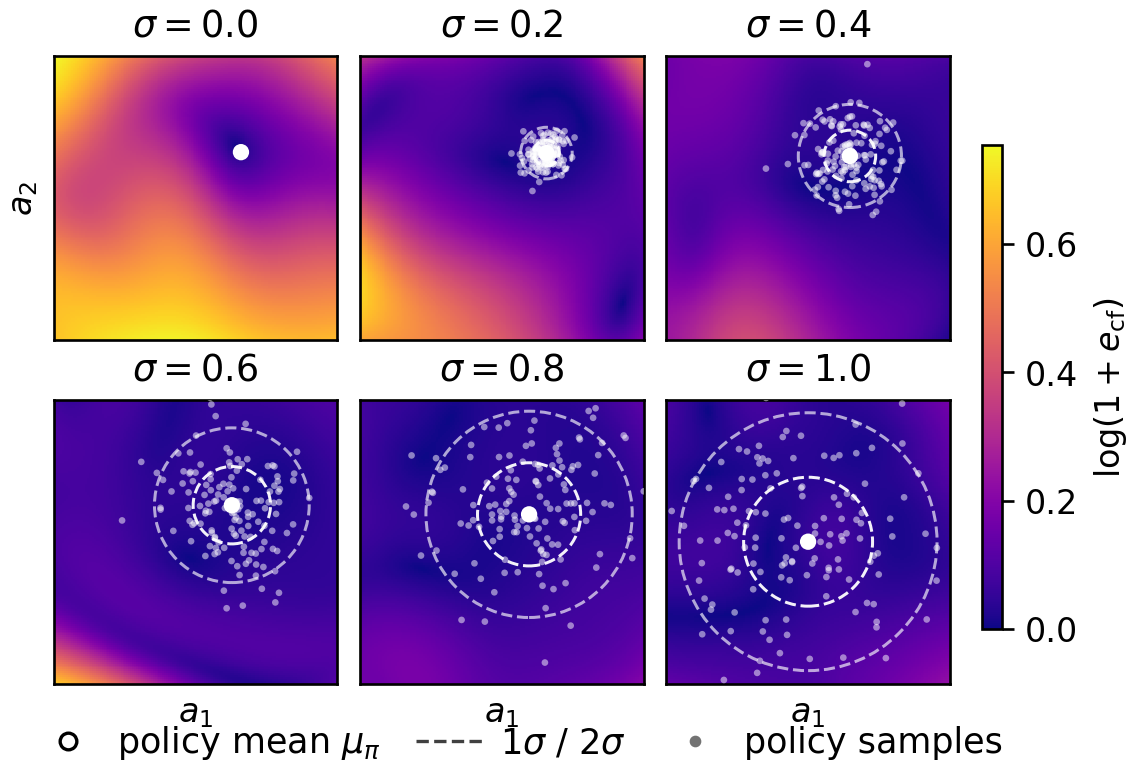}
	\caption{Counterfactual transition error across actions at a fixed state. Predicted next states are orthogonally aligned with the latent coordinates for evaluation.}
	\label{fig:cwm-behavior-coverage}
\end{figure}

\begin{figure}[htbp]
	\centering
	\includegraphics[width=\linewidth]{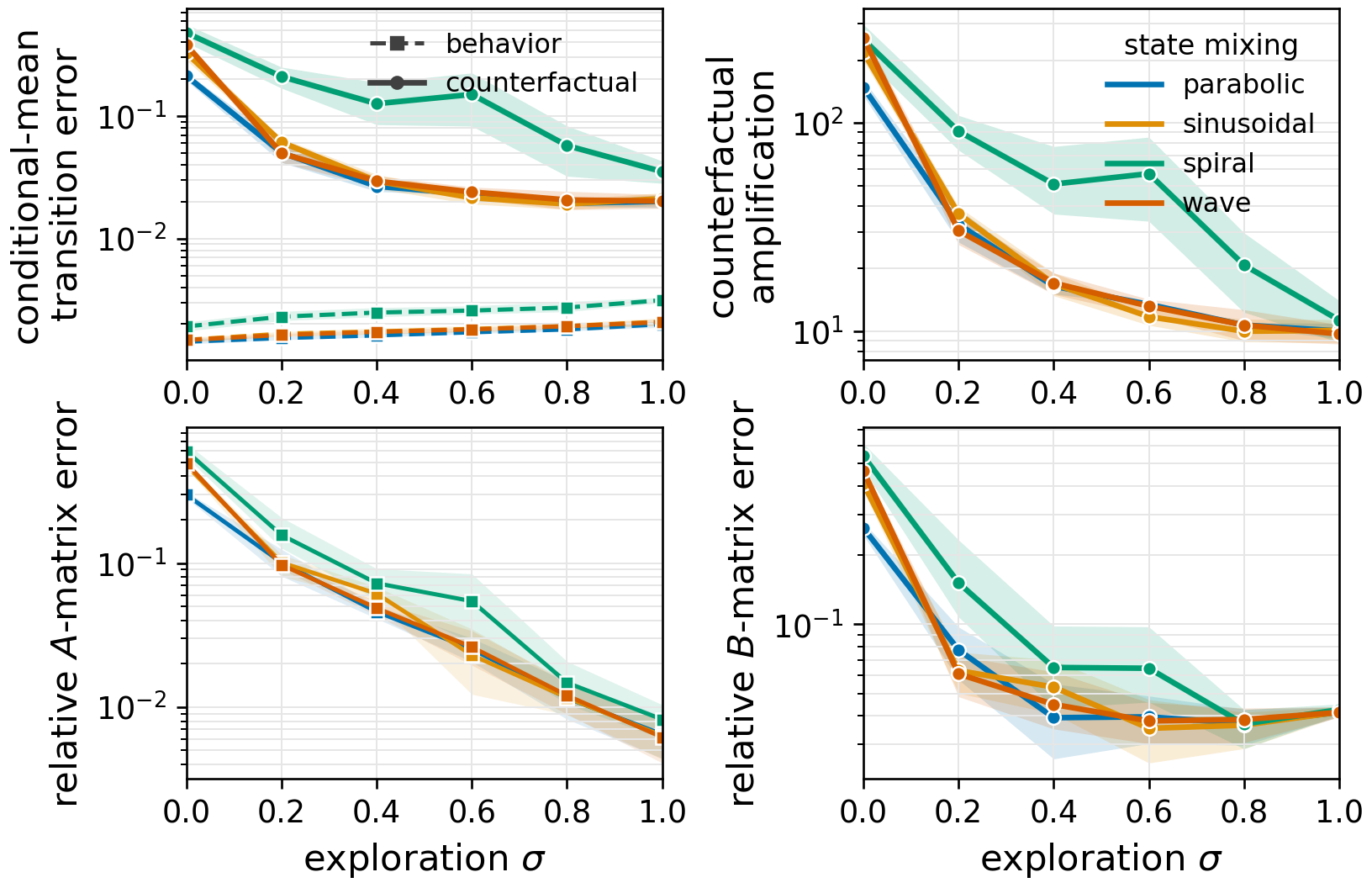}
	\caption{Transition identification as conditional action coverage increases. Top: conditional-mean transition error and empirical counterfactual amplification under a fixed intervention distribution. Bottom: relative estimation errors for the state and action components, represented by matrices \(A\) and \(B\), respectively. Lines and shading show means and 95\% confidence intervals over five runs.}
	\label{fig:cwm-quantitative-summary}
\end{figure}

\subsection{Consequences for Goal-Conditioned Planning}
\label{subsec:planning-experiment}
Counterfactual transition errors matter for planning because the planner must compare action sequences that may depart from the behavior distribution. We test this mechanism with a goal-conditioned planner that encodes the current and goal observations, rolls out a shared bank of bounded action sequences, and selects the sequence whose predicted endpoint is closest to the encoded goal. The selected sequence is then evaluated under the true conditional-mean dynamics. The candidate bank is fixed across models so that performance differences reflect the learned rollouts rather than the search distribution.

The improvement in transition identification carries over to planning performance. Under limited conditional coverage, inaccurate rollouts distort the predicted reachable set and the selected sequence can execute far from its predicted endpoint. As coverage increases, the predicted and true reachable sets become progressively better aligned (Figure~\ref{fig:cwm-planning}). This geometric improvement is reflected in action-selection performance across all observation maps: mean terminal error over a shared set of reachable goals decreases sharply and is nearly eliminated in the well-excited regime (Figure~\ref{fig:cwm-planning-curve}). Averaging over goals rules out the accidental success that can occur for an individual discrete target. These results show that counterfactual transition errors under limited conditional excitation have measurable consequences for control: they distort candidate evaluation and increase the terminal error of model-based action selection.

\begin{figure}[htbp]
	\centering
	\includegraphics[width=\linewidth]{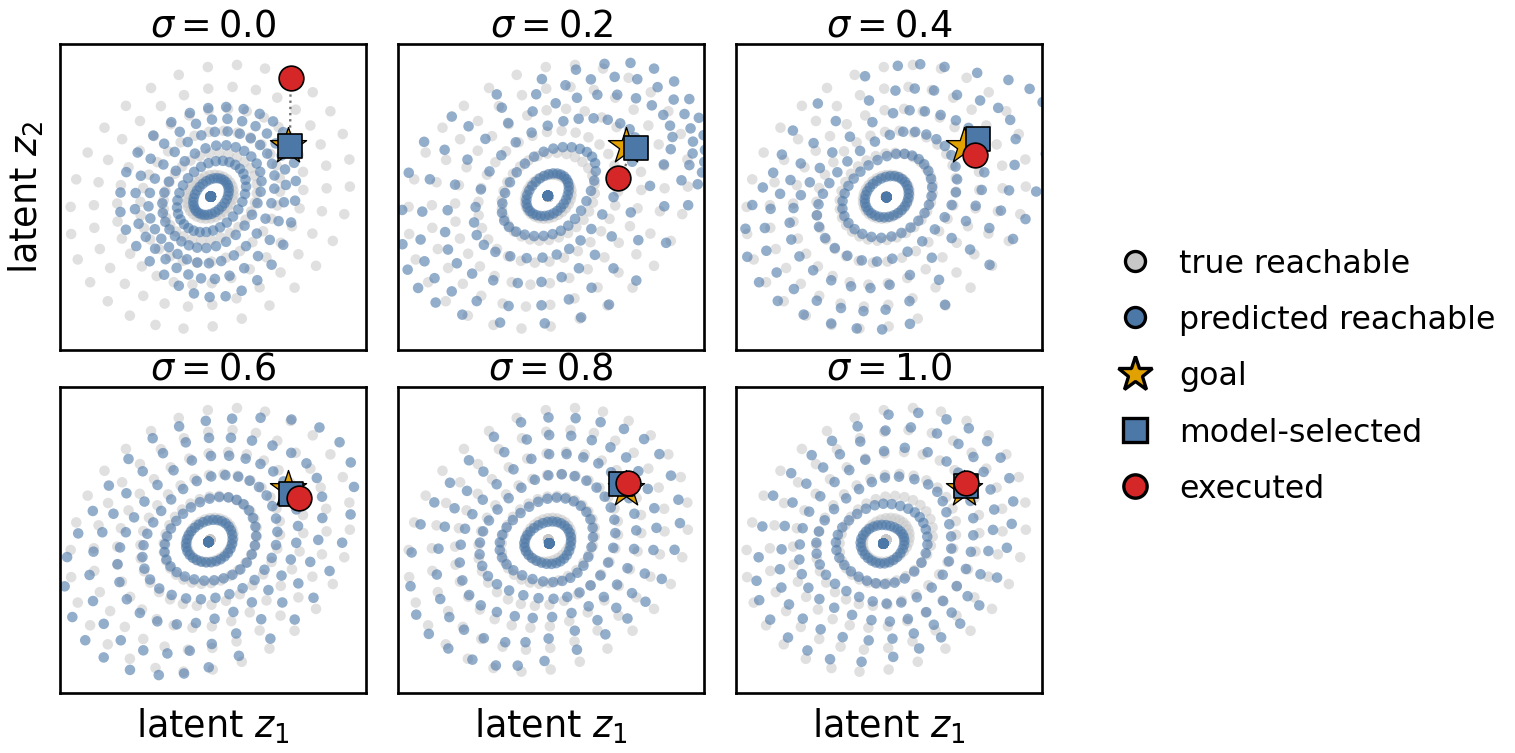}
	\caption{Predicted and true terminal sets across conditional action coverage. Predicted endpoints are orthogonally aligned with the latent coordinates for visualization. The dotted segment indicates the prediction error for the selected sequence.}
	\label{fig:cwm-planning}
\end{figure}

\begin{figure}[htbp]
	\centering
	\includegraphics[width=\linewidth]{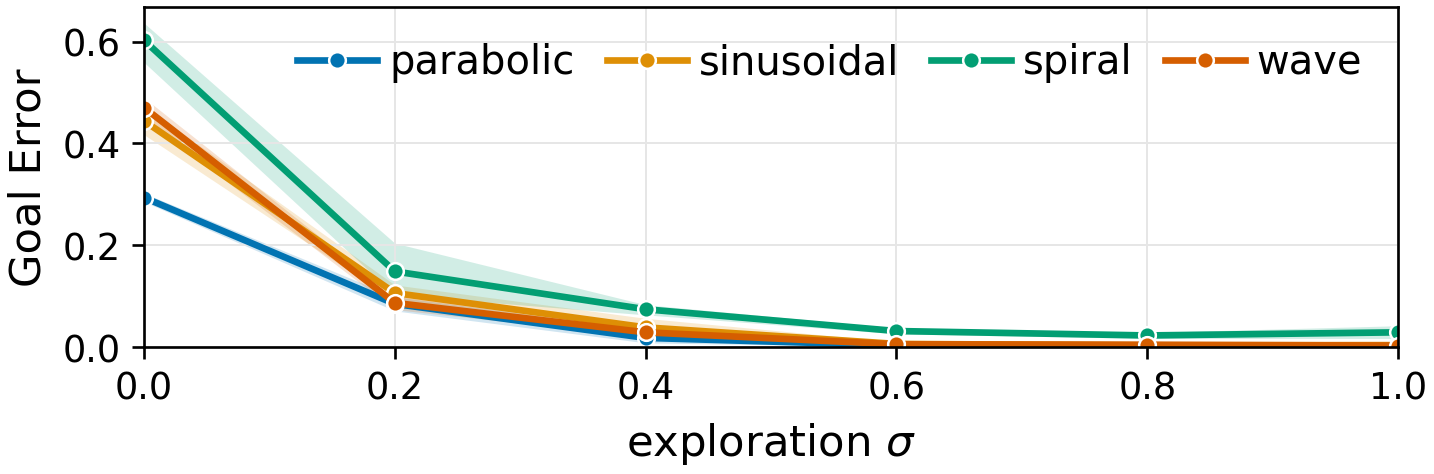}
	\caption{Goal-conditioned planning error as conditional action coverage increases. Curves report mean terminal error across goals reachable by the candidate bank; shading shows 95\% confidence intervals across five runs.}
	\label{fig:cwm-planning-curve}
\end{figure}

	\section{Conclusion}
\label{sec:conclusion}

This paper establishes a joint identifiability theory for controlled world models from nonlinear observations. We show that action-conditioned latent prediction can jointly identify the latent representation up to an orthogonal transformation and the controlled conditional-mean transition, provided that predictable-signal spectral separation and sufficient conditional action excitation hold. In contrast, state-only prediction captures the behavior-policy-averaged future and does not generally identify how alternative actions affect the next state. We further derive quantitative identifiability bounds and show that weak conditional action excitation can make transition identification poorly conditioned, allowing accurate on-policy prediction to coexist with substantially larger counterfactual error. Experiments support the distinct roles of representation and transition identifiability and demonstrate how improved transition identification leads to more reliable latent-space planning. Our analysis is limited to invertible observations, linear Gaussian latent dynamics, an exact representation constraint, and identification of the controlled conditional mean. Future work should extend these guarantees to finite-sample settings and to nonlinear or partially observed dynamics.
\newpage

	\bibliography{main}

\appendix

\section{Proof of Theorem~\ref{thm:state-transition-identifiability}}
\label{app:proof-main}

Throughout the proof, equalities between measurable functions are understood up to the relevant Gaussian measure. 
Let \(w=(z_t^\top,a_t^\top)^\top\), \(C=[A\;B]\), and \(\Sigma_w:=\Cov(w)\). Under Assumptions~1--2, \(w\sim\Normal(0,\Sigma_w)\), \(z_{t+1}=Cw+\xi_t\), and
\[
R_\pi:=\Cov\!\left(\E[z_{t+1}\mid z_t,a_t]\right)=C\Sigma_w C^\top.
\]
The independent-noise special case \(\Sigma_w=I_{d+m}\) gives \(R_\pi=CC^\top=AA^\top+BB^\top\).

\begin{lemma}[Optimal predictor and predictable energy]
\label{lem:optimal_predictor}
For any admissible encoder \(h\), the squared-loss optimal predictor is
\(F_h^\star(h(z_t),a_t)=\E[h(z_{t+1})\mid h(z_t),a_t]\). 
Moreover, minimizing the prediction loss over \(F\) is equivalent to maximizing
\[
	J(h):=\E\!\left[\left\|\E[h(z_{t+1})\mid h(z_t),a_t]\right\|^2\right].
\]
\end{lemma}

\begin{proof}
The first claim follows from the conditional-mean optimality of squared loss. 
By the Pythagorean identity for conditional expectation,
\[
	\inf_F \mathcal L(h,F)
	=
	\E\|h(z_{t+1})\|^2-J(h).
\]
Since \(h\) is standard Gaussian and the process is stationary, \(\E\|h(z_{t+1})\|^2=d\).
Thus minimizing prediction loss is equivalent to maximizing \(J(h)\).
\end{proof}

\begin{lemma}[Controlled Hermite energy bound]
\label{lem:energy_bound}
Let \(h(z)=\sum_{k\ge1}h_k(z)\) be the vector-valued Hermite decomposition under \(z\sim\Normal(0,I_d)\), and write the linear component as \(h_1(z)=Lz\). 
Let \(\tau_1:=\E\|h_1(z)\|^2=\tr(LL^\top)\). 
If the eigenvalues of \(R_\pi\) are \(1>\lambda_1\ge\cdots\ge\lambda_d>0\), then every admissible \(h\) satisfies
\[
	J(h)
	\le
	\tr(R_\pi)
	-
	(\lambda_d-\lambda_1^2)(d-\tau_1).
\]
\end{lemma}

\begin{proof}
Since \(h(z_t)\) is a function of \(z_t\), conditioning on \((z_t,a_t)\) can only increase predictable energy, so
\[
	J(h)
	\le
	\E\!\left[\left\|\E[h(z_{t+1})\mid z_t,a_t]\right\|^2\right].
\]
Write \(w=\Sigma_w^{1/2}u\) for \(u\sim\Normal(0,I_{d+m})\), and set \(M:=C\Sigma_w^{1/2}\). Then \(z_{t+1}=Mu+\xi_t\) is a Gaussian channel with \(MM^\top=R_\pi\). Its canonical correlations are therefore the square roots of the eigenvalues of \(R_\pi\). By the standard Hermite--Mehler decomposition for Gaussian channels, the first-order Hermite component is contracted according to \(R_\pi\), while every Hermite component of degree at least two has predictable energy at most \(\lambda_1^2\) times its variance.
Therefore,
\[
	J(h)
	\le
	\tr(LR_\pi L^\top)+\lambda_1^2(d-\tau_1).
\]
It remains to control the linear term. 
Whitening gives \(I_d=\E[h(z)h(z)^\top]\), and Hermite orthogonality implies \(LL^\top\preceq I_d\). 
Hence
\[
	\tr(LR_\pi L^\top)
	\le
	\tr(R_\pi)-\lambda_d(d-\tau_1).
\]
Combining the two inequalities gives the result.
\end{proof}

\begin{proof}[Proof of Theorem~\ref{thm:state-transition-identifiability}]
For any orthogonal \(Q\in O(d)\), the encoder \(h_Q(z)=Qz\) is admissible. 
Its optimal predictor is \(F_Q(h_Q(z_t),a_t)=Q(Az_t+Ba_t)\), and the corresponding predictable energy is
\[
J(h_Q)=\E\|Q(Az_t+Ba_t)\|^2=\tr(R_\pi).
\]
Thus the maximum value of \(J(h)\) is at least \(\tr(R_\pi)\).

By Lemma~\ref{lem:energy_bound}, every admissible encoder satisfies
\[
	J(h)
	\le
	\tr(R_\pi)-(\lambda_d-\lambda_1^2)(d-\tau_1).
\]
Under the controlled spectral separation condition \(\lambda_d-\lambda_1^2>0\), any global maximizer must have \(\tau_1=d\). 
Since the total Hermite energy of \(h\) is \(d\), all higher-order Hermite components vanish, and hence \(h(z)=Lz\) a.s. 
The standard-Gaussian constraint gives \(LL^\top=I_d\), so \(L\in O(d)\).
Setting \(Q=L\) proves state identifiability.

Finally, for \(y_t=h(z_t)=Qz_t\), the optimal predictor satisfies
\[
	F^\star(y_t,a_t)
	=
	\E[Qz_{t+1}\mid Qz_t,a_t]
	=
	QAQ^\top y_t+QBa_t .
\]
This proves controlled transition identifiability. 
Moreover, \(\rho_{\rm tr}(\pi)>0\) makes the conditional covariance \(\Cov(a_t\mid z_t)\) positive definite. Since \(\Cov(z_t)=I_d\), the Schur-complement criterion implies that \(\Sigma_w\) is positive definite, so the jointly Gaussian behavior distribution has a strictly positive density on \(\mathbb R^{d+m}\). The continuous difference between \(F(Qz,a)\) and \(Q(Az+Ba)\) therefore vanishes everywhere once it vanishes. Equivalently, in identified coordinates, \(F(y,a)=QAQ^\top y+QBa\) for every \((y,a)\in\mathbb R^d\times\mathbb R^m\).
The remaining prediction loss is the irreducible noise variance, namely
\(\tr(\Cov(\xi_t))=d-\tr(R_\pi)\).
\end{proof}
\section{Proof of Theorem~\ref{thm:quantitative-identifiability}}
\label{app:proof-quantitative}

We use the notation of Theorem~\ref{thm:quantitative-identifiability}. 
Let \(R_\pi:=\Cov(Az_t+Ba_t)=C\Sigma_w C^\top\), with eigenvalues \(1>\lambda_1\ge\cdots\ge\lambda_d>0\), and let \(\gamma_{\rm rep}(\pi):=\lambda_d-\lambda_1^2\).
For an encoder satisfying \(h(z_t)\sim\Normal(0,I_d)\), write \(h(z)=Lz+g(z)\) for its Hermite decomposition and \(F_h^\star(h(z_t),a_t):=\E[h(z_{t+1})\mid h(z_t),a_t]\).
Define \(J(h):=\E\|F_h^\star(h(z_t),a_t)\|^2\).

\begin{lemma}[Approximate predictable-energy gap]
\label{lem:approx-energy-gap}
Let \(r:=\E\|g(z)\|^2\).
If \(\mathcal L(h,F_h^\star)\le \mathcal L^\star+\Delta_{\rm enc}\), then
\[
    r
    \le
    C\frac{\Delta_{\rm enc}}{\gamma_{\rm rep}(\pi)} .
\]
\end{lemma}

\begin{proof}
Since \(h(z_t)\sim\Normal(0,I_d)\), \(\mathcal L(h,F_h^\star)=d-J(h)\). The excess-risk assumption therefore gives \(J(h)\ge \tr(R_\pi)-\Delta_{\rm enc}\).

On the other hand, the controlled Hermite contraction used in the exact proof gives
\[
    J(h)
    \le
    \tr(LR_\pi L^\top)+\lambda_1^2 r .
\]
\[
    \tr(LR_\pi L^\top)
    \le
    \tr(R_\pi)-\lambda_d r.
\]
Combining the two displays yields
\((\lambda_d-\lambda_1^2)r\le \Delta_{\rm enc}\), which proves the claim.
\end{proof}

\begin{lemma}[Approximate orthogonal identification]
\label{lem:approx-orthogonal-identification}
Under the assumptions of Lemma~\ref{lem:approx-energy-gap}, there exists \(Q\in O(d)\) such that
\[
    \E\|h(z_t)-Qz_t\|^2
    \le
    C\frac{\Delta_{\rm enc}}{\gamma_{\rm rep}(\pi)} .
\]
\end{lemma}

\begin{proof}
Since \(h(z_t)\sim\Normal(0,I_d)\), Hermite orthogonality gives
\[
    I_d=LL^\top+\Cov(g(z)).
\]
Therefore
\[
    \|LL^\top-I_d\|_F
    =
    \|\Cov(g(z))\|_F
    \le
    r .
\]
By stability of the polar decomposition, there exists \(Q\in O(d)\) such that \(\|L-Q\|_F^2\le Cr\).
Using again the orthogonality between the linear and higher-order Hermite components,
\[
    \E\|h(z_t)-Qz_t\|^2
    =
    \|L-Q\|_F^2+\E\|g(z_t)\|^2
    \le
    Cr.
\]
The result follows from Lemma~\ref{lem:approx-energy-gap} and \(\gamma_{\rm rep}(\pi)\le1\).
\end{proof}

\begin{proof}[Proof of Theorem~\ref{thm:quantitative-identifiability}]
Let \(\eta:=\Delta_{\rm enc}/\gamma_{\rm rep}(\pi)\).
Lemma~\ref{lem:approx-orthogonal-identification} gives an orthogonal matrix \(Q\in O(d)\) such that \(\E\|h(z_t)-Qz_t\|^2\le C\eta\). 
This proves the state-identification bound.

It remains to prove the transition-identification bound. 
Let \(e(z):=h(z)-Qz\). 
For the optimal predictor \(F_h^\star\), we decompose
\begin{equation}
\begin{aligned}
    F_h^\star(h(z_t),a_t)
    &-Q(Az_t+Ba_t)\\
    &=\E[e(z_{t+1})\mid h(z_t),a_t]\\
    &\quad+\E[QAz_t\mid h(z_t),a_t]-QAz_t.
\end{aligned}
\end{equation}
The first term is bounded by Jensen's inequality and stationarity:
\[
    \E\|\E[e(z_{t+1})\mid h(z_t),a_t]\|^2
    \le
    \E\|e(z_{t+1})\|^2
    =
    \E\|e(z_t)\|^2 .
\]
For the second term, conditional expectation is the best \(L^2\) approximation by functions of \((h(z_t),a_t)\). 
Using \(QAQ^\top h(z_t)\) as a candidate, and writing
\(r_t:=\E[QAz_t\mid h(z_t),a_t]-QAz_t\), gives
\begin{equation}
\begin{aligned}
    \E\|r_t\|^2
    &\le \|A\|_{\rm op}^2
    \E\!\left\|h(z_t)-Qz_t\right\|^2 .
\end{aligned}
\end{equation}
Thus
\[
    \E\|F_h^\star(h(z_t),a_t)-Q(Az_t+Ba_t)\|^2
    \le
    C(1+\|A\|_{\rm op}^2)\eta .
\]
Finally, for a learned predictor \(F\) satisfying \(\E\|F(h(z_t),a_t)-F_h^\star(h(z_t),a_t)\|^2\le\Delta_{\rm pred}\), the triangle inequality yields
\begin{equation}
\begin{aligned}
    \E\!\left\|F(h(z_t),a_t)-Q(Az_t+Ba_t)\right\|^2
    &\le C\Delta_{\rm pred}\\
    &\quad+C(1+\|A\|_{\rm op}^2)\eta .
\end{aligned}
\end{equation}
This completes the proof.
\end{proof}

\section{Proof of Theorem~\ref{thm:counterfactual-amplification}}
\label{app:proof-counterfactual-amplification}

\begin{proof}
		Let \(Y_t=Qz_t\). Under the identified coordinates,
	\[
	Y_{t+1}
	=
	QAz_t+QBa_t+Q\xi_t
	=
	F^\star(Y_t,a_t)+Q\xi_t.
	\]
	Since \(\xi_t\) is independent of \((z_t,a_t)\) and has zero mean, \(F^\star(Y_t,a_t)\) is the conditional mean of \(Y_{t+1}\) given \((Y_t,a_t)\).
	
	Consequently, the conditional-mean orthogonality identity gives, for every square-integrable predictor \(F\). Let
	\(d_F(y,a):=F(y,a)-F^\star(y,a)\).
	\begin{align}
		\mathcal E_\pi(F)
		&=
		\E_{P_\pi}
		\left[
		\|d_F(Y_t,a_t)\|^2
		\right].
		\label{eq:conditional-mean-orthogonality}
	\end{align}
	Indeed, expanding the squared error yields
	\begin{align*}
		&\E_{P_\pi}
		\left[
		\|Y_{t+1}-F(Y_t,a_t)\|^2
		\right] \\
		&\quad =
		\E_{P_\pi}
		\left[
		\|Y_{t+1}-F^\star(Y_t,a_t)\|^2
		\right]
		+
		\E_{P_\pi}
		\left[
		\|d_F(Y_t,a_t)\|^2
		\right],
	\end{align*}
	because the cross term vanishes after conditioning on \((Y_t,a_t)\).
	
	Let \(v_{\min}\in\R^m\) be a unit eigenvector of
	\(\Sigma_{\mathrm{tr}}(\pi)\) associated with its smallest eigenvalue
	\(\rho_{\mathrm{tr}}(\pi)>0\), and let \(u\in\R^d\) be any unit vector. Define
	\[
	D_\delta
	:=
	\sqrt{\frac{\delta}{\rho_{\mathrm{tr}}(\pi)}}
	\,u v_{\min}^\top
	\]
	and construct
	\[
	F_\delta(y,a)
	:=
	F^\star(y,a)
	+
	D_\delta
	\bigl(
	a-\mu_\pi(Q^\top y)
	\bigr).
	\]
	Because \((z_t,a_t)\) is jointly Gaussian, the conditional mean
	\(\mu_\pi(z)=\E[a_t\mid z_t=z]\) is linear in \(z\). Hence \(F_\delta\) is continuous and square-integrable.
	
	Let \(r_t:=a_t-\mu_\pi(z_t)\). By definition,
	\(\E[r_t\mid z_t]=0\), and
	\[
	\E[r_t r_t^\top]
	=
	\E_{z_t}
	\left[
	\Cov_\pi(a_t\mid z_t)
	\right]
	=
	\Sigma_{\mathrm{tr}}(\pi).
	\]
	Using Eq.~\eqref{eq:conditional-mean-orthogonality},
	\begin{align*}
		\mathcal E_\pi(F_\delta)
		&=
		\E
		\left[
		\|D_\delta r_t\|^2
		\right] \\
		&=
		\tr
		\left(
		D_\delta
		\Sigma_{\mathrm{tr}}(\pi)
		D_\delta^\top
		\right) \\
		&=
		\frac{\delta}{\rho_{\mathrm{tr}}(\pi)}
		\tr
		\left(
		u v_{\min}^\top
		\Sigma_{\mathrm{tr}}(\pi)
		v_{\min}u^\top
		\right) \\
		&=
		\frac{\delta}{\rho_{\mathrm{tr}}(\pi)}
		\rho_{\mathrm{tr}}(\pi)
		\|u\|^2
		=
		\delta.
	\end{align*}
	
	Under the counterfactual distribution,
	\(a_{\rm cf}-\mu_\pi(z)=\eta_{\rm cf}\) with
	\(\eta_{\rm cf}\sim\Normal(0,I_m)\). Therefore,
	\begin{align*}
		\mathcal E_{\rm cf}(F_\delta)
		&=
		\E
		\left[
		\|D_\delta\eta_{\rm cf}\|^2
		\right] \\
		&=
		\tr(D_\delta D_\delta^\top) \\
		&=
		\|D_\delta\|_F^2 \\
		&=
		\frac{\delta}{\rho_{\mathrm{tr}}(\pi)}
		\|u\|^2
		\|v_{\min}\|^2 \\
		&=
		\frac{\delta}{\rho_{\mathrm{tr}}(\pi)}.
	\end{align*}
	This proves the result.
\end{proof}

\begin{proof}[Proof of Corollary~\ref{cor:exploration-noise-scaling}]
	For the behavior policy
	\[
	a_t
	=
	\sqrt{1-\sigma^2}\,Kz_t
	+
	\sigma\eta_t,
	\]
	the conditional mean is
	\(\mu_{\pi_\sigma}(z_t)=\sqrt{1-\sigma^2}\,Kz_t\). Since
	\(\eta_t\sim\Normal(0,I_m)\) is independent of \(z_t\),
	\[
	\Cov_{\pi_\sigma}(a_t\mid z_t)
	=
	\sigma^2I_m.
	\]
	It follows immediately that
	\[
	\Sigma_{\mathrm{tr}}(\pi_\sigma)
	=
	\E_{z_t}
	\left[
	\Cov_{\pi_\sigma}(a_t\mid z_t)
	\right]
	=
	\sigma^2I_m
	\]
	and hence
	\[
	\rho_{\mathrm{tr}}(\pi_\sigma)
	=
	\lambda_{\min}
	\left(
	\Sigma_{\mathrm{tr}}(\pi_\sigma)
	\right)
	=
	\sigma^2.
	\]
	For every \(\sigma>0\), applying
	Theorem~\ref{thm:counterfactual-amplification} gives
	\(\mathcal E_{\rm cf}(F_\delta)=\delta/\sigma^2\).
	
	When \(\sigma=0\), the action is the deterministic function \(a_t=Kz_t\), so \(a_t-\mu_{\pi_0}(z_t)=0\). Let \(u\in\R^d\) and \(v\in\R^m\) be arbitrary unit vectors, and define
	\[
	\widetilde F(y,a)
	:=
	F^\star(y,a)
	+
	uv^\top
	\bigl(
	a-\mu_{\pi_0}(Q^\top y)
	\bigr).
	\]
	Then \(\widetilde F=F^\star\) under \(P_{\pi_0}\), and therefore
	\(\mathcal E_{\pi_0}(\widetilde F)=0\). Under the counterfactual action distribution,
	\[
	\mathcal E_{\rm cf}(\widetilde F)
	=
	\E
	\left[
	\|uv^\top\eta_{\rm cf}\|^2
	\right]
	=
	1.
	\]
	Thus distinct continuous predictors attain the same on-policy risk while producing different counterfactual predictions, proving structural non-identifiability at \(\sigma=0\).
\end{proof}
\section{Experimental Details}
\label{app:experimental-details}

\paragraph{Model, optimization, and evaluation.}
All experiments use two-dimensional latent states, observations, and actions. The encoder and action-conditioned predictor use 8-layer SiLU MLPs with hidden width 512. Both representation and coverage runs use \(100{,}000\) one-step transitions, \(5{,}000\) AdamW updates, gradient clipping at 1.0, and a 500-step warmup followed by cosine decay. They use batch size 1024, learning rate \(10^{-3}\), and weight decay \(10^{-5}\). The objective is
\begin{equation}
\begin{aligned}
    \mathcal L
    &=\mathcal L_{\rm pred}
    +\lambda_{\rm white}\mathcal L_{\rm white}
    +\lambda_{\rm gauss}\mathcal L_{\rm gauss},\\
    \lambda_{\rm white}&=1,\qquad
    \lambda_{\rm gauss}=50.
\end{aligned}
\end{equation}
where \(\mathcal L_{\rm gauss}\) is a sliced characteristic-function regularizer toward a standard Gaussian. Each run uses \(10{,}000\) evaluation transitions and a separate set of \(5{,}000\) states to fit the orthogonal Procrustes alignment \(Q^\star\). Both sweeps use five independent runs for each observation map.

For \(y_i=h(g(z_i))\), the normalized representation and controlled-transition errors are
\begin{equation}
    \begin{aligned}
        E_{\rm rep}
        &:=
        \frac{\sum_i\|y_i-Q^\star z_i\|^2}
             {\sum_i\|z_i\|^2},\\
        E_{\rm tr}
        &:=
        \frac{\sum_i\|F(y_i,a_i)-Q^\star(Az_i+Ba_i)\|^2}
             {\sum_i\|Az_i+Ba_i\|^2}.
    \end{aligned}
    \label{eq:appendix-experimental-errors}
\end{equation}
All transition metrics compare predictions with the noise-free conditional mean, not with sampled next states.

\subsection{Representation Sweep}
\label{app:representation-details}

Let \(R_\theta\) denote a two-dimensional rotation and define
\begin{equation}
    \begin{aligned}
        A(q)&=R_{30^\circ}D(q)R_{20^\circ}^\top,\\
        B(q)&=R_{30^\circ}D(q)R_{-20^\circ}^\top,\\
        D(q)&=\diag\!\left(\sqrt{0.30},\sqrt{q/2}\right).
    \end{aligned}
    \label{eq:appendix-representation-dynamics}
\end{equation}
Actions are independent standard Gaussian, and the process covariance is
\(\Cov(\xi_t)=I-AA^\top-BB^\top\). Consequently, \(z_t\), \(a_t\), and \(z_{t+1}\) are standard Gaussian in marginal distribution, while
\[
    \operatorname{spec}(AA^\top+BB^\top)=\{0.6,q\},
    \qquad
    \gamma_{\rm rep}(q)=q-0.36.
\]
We use
\[
    q\in\{0.1, 0.2, 0.3, 0.36,0.4, 0.5, 0.6\}.
\]

The four observation maps are:
\begin{align*}
    \text{spiral:}\quad
        &u=R_{0.8\pi\|z\|}z,\\[-2pt]
        &g(z)=(1+0.45\tanh(u_1))u,\\
    \text{parabolic:}\quad
        &g(z)=(z_1,\ z_2+0.75z_1^2),\\
    \text{sinusoidal:}\quad
        &g(z)=(z_1+1.75\sin(1.5z_2),\ z_2),\\
    \text{wave:}\quad
        &u_1=z_1+0.65\sin(1.4z_2),\\[-2pt]
        &u_2=z_2+0.55\sin(1.2u_1+0.4),\\[-2pt]
        &g(z)=(u_1+0.40\sin(1.7u_2-0.3),\ u_2).
\end{align*}
Figure~\ref{fig:cwm-representation-geometry} uses \(q=0.50\).

\subsection{Coverage Sweep}
\label{app:coverage-details}

We use the state-dependent Gaussian behavior policy
\begin{equation}
\begin{aligned}
    a_t&=\sqrt{1-\sigma^2}\,Kz_t+\sigma\eta_t,\\
    &\eta_t\sim\Normal(0,I),\quad \eta_t\perp z_t,\quad KK^\top=I.
\end{aligned}
\label{eq:behavior-policy}
\end{equation}
The coverage levels and feedback matrix are
\[
    \sigma\in\{0,0.2,0.4,0.6,0.8,1\},
    \qquad
    K=R_{45^\circ},
\]
We reuse Eq.~\eqref{eq:appendix-representation-dynamics} with \(q=0.50\) and run the coverage sweep independently for all four observation maps listed above. Under the state-dependent policy in Eq.~\eqref{eq:behavior-policy}, the predictable covariance is
\begin{equation}
    R_\pi(\sigma)
    =
    AA^\top+BB^\top
    +\sqrt{1-\sigma^2}
      \left(AK^\top B^\top+BKA^\top\right).
    \label{eq:appendix-coverage-covariance}
\end{equation}
For each \(\sigma\), we set the independent zero-mean process-noise covariance to
\[
    \Cov(\xi_t)=I-R_\pi(\sigma)
\]
so that current and next states have the same standard-Gaussian marginal. The controlled conditional mean \(Az+Ba\), observation map, and action marginal are fixed across the sweep, but the full stochastic transition kernel is therefore indexed by \(\sigma\) through this stationarity-preserving noise adjustment. The process-noise trace changes mildly, from \(0.8\) at \(\sigma=0\) to \(0.9\) at \(\sigma=1\). The representation margin \(\gamma_{\rm rep}(\sigma)=\lambda_{\min}(R_\pi(\sigma))-\lambda_{\max}(R_\pi(\sigma))^2\) ranges from \(0.118\) to \(0.140\) and is not used as a training or model-selection gate.

The counterfactual actions are sampled once from \(\Normal(0,2^2I)\) and shared across all coverage levels and training runs. This intervention probes weakly covered parts of the action plane while keeping the evaluation distribution fixed across the sweep. For held-out states \(z_i\), let \(y_i=h(g(z_i))\), let \(a_i^{\rm beh}\) denote the corresponding behavior action, and let \(a_i^{\rm cf}\sim\Normal(0,2^2I)\) denote the counterfactual action. The per-coordinate transition errors are
\begin{equation}
\begin{aligned}
    E_{\rm trans}^{\rm beh}
    &=\frac{1}{nd}\sum_{i=1}^{n}
      \bigl\|F(y_i,a_i^{\rm beh})-Q^\star(Az_i+Ba_i^{\rm beh})\bigr\|^2,\\
    E_{\rm trans}^{\rm cf}
    &=\frac{1}{nd}\sum_{i=1}^{n}
      \bigl\|F(y_i,a_i^{\rm cf})-Q^\star(Az_i+Ba_i^{\rm cf})\bigr\|^2.
\end{aligned}
\label{eq:appendix-coverage-errors}
\end{equation}
Their ratio \(E_{\rm trans}^{\rm cf}/E_{\rm trans}^{\rm beh}\) is the empirical counterfactual amplification.

To assess identification of the state and action components separately, we sample independent probe actions \(a_i^{\rm probe}\sim\Normal(0,I)\) and fit
\[
    F(y_i,a_i^{\rm probe})
    \approx \widehat M y_i+\widehat N a_i^{\rm probe}
\]
by least squares. The relative component error is
\begin{equation}
\begin{aligned}
    E_T&=\frac{\|\widehat T-T^\star\|_F}{\|T^\star\|_F},\\
    (\widehat T,T^\star)&\in
    \left\{\left(\widehat M,Q^\star A Q^{\star\top}\right),
    \left(\widehat N,Q^\star B\right)\right\}.
\end{aligned}
\label{eq:appendix-component-errors}
\end{equation}
The counterfactual field in Figure~\ref{fig:cwm-behavior-coverage} uses \(z_0=(1,0)\) and an action grid over \([-2.2,2.2]^2\).

\subsection{Planning Probe}
\label{app:planning-details}

The planning visualization uses the seed-0 spiral-map model at each coverage level. We set \(z_{\rm init}=(0,0)\), horizon \(H=6\), and visualization goal \(z_g=(0.85,0.55)\). The current and goal observations are encoded as \(y_{\rm init}=h(g(z_{\rm init}))\) and \(y_g=h(g(z_g))\). The common candidate bank contains 210 constant-action sequences: seven radii uniformly spaced from \(0\) to \(1.4\), crossed with 30 uniformly spaced directions. For each candidate, the learned model is rolled forward from \(y_{\rm init}\), and the selected sequence minimizes \(\|\widehat y_H-y_g\|\). Thus neither the true latent coordinates nor \(Q^\star\) enter action selection. After selection, \(Q^{\star\top}\) is used only to display predicted endpoints in the latent coordinate system. The executed endpoint is obtained from the true noise-free recursion \(z_{t+1}=Az_t+Ba_t\).

The aggregate evaluation uses all five seeds for all four observation maps. Each of the 180 nonzero endpoints reachable by the common candidate bank is used once as a goal, and we report the mean executed terminal error over this goal set for each checkpoint. Because every goal is generated by a candidate in the bank, the candidate-bank oracle error is zero. Figure~\ref{fig:cwm-planning-curve} reports the mean and bootstrap 95\% confidence interval across seeds.

\end{document}